%% file: main.tex
\documentclass{article}

\usepackage[T1]{fontenc}
\usepackage[latin9]{inputenc}
\usepackage{babel}
\usepackage{textcomp}
\usepackage{amsmath}
\usepackage{amsthm}
\usepackage{amssymb}

\makeatletter
\usepackage{wrapfig}
\usepackage{microtype}
\usepackage{graphicx}
\usepackage{subfigure}
\usepackage{booktabs} % for professional tables
\usepackage{xcolor}
\usepackage{bbm}

\usepackage{mathtools}
\usepackage{natbib}
\usepackage{verbatim}
\usepackage[ruled,vlined]{algorithm2e}
\usepackage{algorithmic}
\usepackage{enumitem}
\setlist[itemize]{topsep=0pt, leftmargin=3mm}

\newtheorem{theorem}{Theorem}
\newtheorem{lemma}[theorem]{Lemma}

\usepackage{tikz}
\usetikzlibrary{intersections,shapes.arrows}

\usepackage{amsthm,amsmath,amssymb,amsfonts}

\newcommand{\var}[1]{\text{Var}\left( #1 \right)}

\DeclareMathOperator*{\argmax}{arg\,max}

\usepackage{hyperref}

\usepackage[nonatbib, final]{neurips_2020}
\usepackage{comment}

\title{Provably Efficient Online Hyperparameter Optimization with Population-Based Bandits}

\author{%
  Jack Parker-Holder \\
  University of Oxford\\
  \texttt{jackph@robots.ox.ac.uk} \\
  \And
  Vu Nguyen \\
  University of Oxford\\
  \texttt{vu@robots.ox.ac.uk} \\
  \And
  Stephen J. Roberts \\
  University of Oxford\\
  \texttt{sjrob@robots.ox.ac.uk} \\
}

\begin{document}

\maketitle
\begin{abstract}
Many of the recent triumphs in machine learning are dependent on well-tuned hyperparameters. This is particularly prominent in reinforcement learning (RL) where a small change in the configuration can lead to failure. Despite the importance of tuning hyperparameters, it remains expensive and is often done in a naive and laborious way. A recent solution to this problem is Population Based Training (PBT) which  updates both weights and hyperparameters in a \emph{single training run} of a population of agents. PBT has been shown to be particularly effective in RL, leading to widespread use in the field. However, PBT lacks theoretical guarantees since it relies on random heuristics to explore the hyperparameter space. This inefficiency means it typically requires vast computational resources, which is prohibitive for many small and medium sized labs. In this work, we introduce the first provably efficient PBT-style algorithm, Population-Based Bandits (PB2). PB2 uses a probabilistic model to guide the search in an efficient way, making it possible to discover high performing hyperparameter configurations with far fewer agents than typically required by PBT. We show in a series of RL experiments that PB2 is able to achieve high performance with a modest computational budget. 

%\vcom{optional title 1: Population-based Bandits}

%\vcom{optional title 2: Gaussian Process Population-based Bandits}

\end{abstract}

\input{macros.tex}

\input{intro.tex}

\input{preliminaries.tex}

\input{algs.tex}

\input{related_work.tex}

\input{experiments.tex}

\input{conclusions.tex}

\section*{Disclosure of Funding}

Nothing to declare.

\section*{Broader Impact}
Population Based Training (PBT) has become a prominent algorithm in machine learning research, leading to gains in reinforcement learning (e.g. IMPALA \cite{impala}) and industrial applications (e.g. \href{https://deepmind.com/blog/article/how-evolutionary-selection-can-train-more-capable-self-driving-cars}{Waymo}). As such, we believe the gains provided by PB2 will have a significant impact. We believe this work will allow labs with small to medium sized computational resources to gain the benefit of population-based training without the excessive computational cost required to ensure sufficient exploration. This should be particularly helpful for achieving competitive performance in reinforcement learning experiments. To aid this, our implementation is integrated with the widely used ray library \cite{liaw2018tune}. 

\bibliographystyle{abbrv}
\bibliography{refs}

\newpage

\input{appendix.tex}

\input{TVGP_Theoretical_Analysis_2.tex}

\end{document}

%% file: macros.tex
\global\long\def\se{\hat{\text{se}}}%

\global\long\def\interior{\text{int}}%

\global\long\def\boundary{\text{bd}}%

\global\long\def\new{\text{*}}%

\global\long\def\stir{\text{Stirl}}%

\global\long\def\dist{d}%

\global\long\def\HX{\entro\left(X\right)}%
 
\global\long\def\entropyX{\HX}%

\global\long\def\HY{\entro\left(Y\right)}%
 
\global\long\def\entropyY{\HY}%

\global\long\def\HXY{\entro\left(X,Y\right)}%
 
\global\long\def\entropyXY{\HXY}%

\global\long\def\mutualXY{\mutual\left(X;Y\right)}%
 
\global\long\def\mutinfoXY{\mutualXY}%

\global\long\def\xnew{y}%

\global\long\def\bx{\mathbf{x}}%

\global\long\def\bz{\mathbf{z}}%

\global\long\def\bu{\mathbf{u}}%

\global\long\def\bs{\boldsymbol{s}}%

\global\long\def\bk{\mathbf{k}}%

\global\long\def\bX{\mathbf{X}}%

\global\long\def\tbx{\tilde{\bx}}%

\global\long\def\by{\mathbf{y}}%

\global\long\def\bY{\mathbf{Y}}%

\global\long\def\bZ{\boldsymbol{Z}}%

\global\long\def\bU{\boldsymbol{U}}%

\global\long\def\bv{\boldsymbol{v}}%

\global\long\def\bn{\boldsymbol{n}}%

\global\long\def\bV{\boldsymbol{V}}%

\global\long\def\bK{\boldsymbol{K}}%

\global\long\def\bw{\vt w}%

\global\long\def\bbeta{\mathbf{\boldsymbol{\beta}}}%

\global\long\def\bmu{\gvt{\mu}}%

\global\long\def\btheta{\boldsymbol{\theta}}%

\global\long\def\blambda{\boldsymbol{\lambda}}%

\global\long\def\bgamma{\boldsymbol{\gamma}}%

\global\long\def\bpsi{\boldsymbol{\psi}}%

\global\long\def\bphi{\boldsymbol{\phi}}%

\global\long\def\bpi{\boldsymbol{\pi}}%

\global\long\def\eeta{\boldsymbol{\eta}}%

\global\long\def\bomega{\boldsymbol{\omega}}%

\global\long\def\bepsilon{\boldsymbol{\epsilon}}%

\global\long\def\btau{\boldsymbol{\tau}}%

\global\long\def\bSigma{\gvt{\Sigma}}%

\global\long\def\realset{\mathbb{R}}%

\global\long\def\realn{\realset^{n}}%

\global\long\def\integerset{\mathbb{Z}}%

\global\long\def\natset{\integerset}%

\global\long\def\integer{\integerset}%

\global\long\def\natn{\natset^{n}}%

\global\long\def\rational{\mathbb{Q}}%

\global\long\def\rationaln{\rational^{n}}%

\global\long\def\complexset{\mathbb{C}}%

\global\long\def\comp{\complexset}%

\global\long\def\compl#1{#1^{\text{c}}}%

\global\long\def\and{\cap}%

\global\long\def\compn{\comp^{n}}%

\global\long\def\comb#1#2{\left({#1\atop #2}\right) }%

\global\long\def\nchoosek#1#2{\left({#1\atop #2}\right)}%

\global\long\def\param{\vt w}%

\global\long\def\Param{\Theta}%

\global\long\def\meanparam{\gvt{\mu}}%

\global\long\def\meanmap{\mathbf{m}}%

\global\long\def\logpart{A}%

\global\long\def\simplex{\Delta}%

\global\long\def\simplexn{\simplex^{n}}%

\global\long\def\dirproc{\text{DP}}%

\global\long\def\ggproc{\text{GG}}%

\global\long\def\DP{\text{DP}}%

\global\long\def\ndp{\text{nDP}}%

\global\long\def\hdp{\text{HDP}}%

\global\long\def\gempdf{\text{GEM}}%

\global\long\def\ei{\text{EI}}%

\global\long\def\rfs{\text{RFS}}%

\global\long\def\bernrfs{\text{BernoulliRFS}}%

\global\long\def\poissrfs{\text{PoissonRFS}}%

\global\long\def\grad{\gradient}%
 
\global\long\def\gradient{\nabla}%

\global\long\def\cpr#1#2{\Pr\left(#1\ |\ #2\right)}%

\global\long\def\var{\text{Var}}%

\global\long\def\Var#1{\text{Var}\left[#1\right]}%

\global\long\def\cov{\text{Cov}}%

\global\long\def\Cov#1{\cov\left[ #1 \right]}%

\global\long\def\COV#1#2{\underset{#2}{\cov}\left[ #1 \right]}%

\global\long\def\corr{\text{Corr}}%

\global\long\def\sst{\text{T}}%

\global\long\def\SST{\sst}%

\global\long\def\ess{\mathbb{E}}%

\global\long\def\Ess#1{\ess\left[#1\right]}%

%\newcommandx\ESS[2][usedefault, addprefix=\global, 1=]{\underset{#2}{\ess}\left[#1\right]}%

\global\long\def\fisher{\mathcal{F}}%

\global\long\def\bfield{\mathcal{B}}%
 
\global\long\def\borel{\mathcal{B}}%

\global\long\def\bernpdf{\text{Bernoulli}}%

\global\long\def\betapdf{\text{Beta}}%

\global\long\def\dirpdf{\text{Dir}}%

\global\long\def\gammapdf{\text{Gamma}}%

\global\long\def\gaussden#1#2{\text{Normal}\left(#1, #2 \right) }%

\global\long\def\gauss{\mathbf{N}}%

\global\long\def\gausspdf#1#2#3{\text{Normal}\left( #1 \lcabra{#2, #3}\right) }%

\global\long\def\multpdf{\text{Mult}}%

\global\long\def\poiss{\text{Pois}}%

\global\long\def\poissonpdf{\text{Poisson}}%

\global\long\def\pgpdf{\text{PG}}%

\global\long\def\iwshpdf{\text{InvWish}}%

\global\long\def\nwpdf{\text{NW}}%

\global\long\def\niwpdf{\text{NIW}}%

\global\long\def\studentpdf{\text{Student}}%

\global\long\def\unipdf{\text{Uni}}%

\global\long\def\transp#1{\transpose{#1}}%
 
\global\long\def\transpose#1{#1^{\mathsf{T}}}%

\global\long\def\mgt{\succ}%

\global\long\def\mge{\succeq}%

\global\long\def\idenmat{\mathbf{I}}%

\global\long\def\trace{\mathrm{tr}}%

%% file: intro.tex
\section{Introduction}

Deep neural networks \cite{deeplearning, lstm, alexnet} have achieved remarkable success in a variety of fields. Some of the most notable results have come in reinforcement learning (RL), where the last decade has seen a series of significant achievements in games \cite{alphago, dqn, dota} and robotics \cite{dexterous_openai, qt_opt}. However, it is notoriously difficult to reproduce RL results, often requiring excessive trial-and-error to find the optimal hyperparameter configurations \cite{hpo_reproducible, deeprlmatters,boil}. 
%Furthermore, with costly experiments, the price paid for inefficient search can be prohibitive.

This has led to a surge in popularity for Automated Machine Learning (AutoML, \cite{automl_book}), which seeks to automate the training of machine learning models. A key component in AutoML is automatic hyperparameter selection \cite{hyperparam_opt, hpo_nlp}, where popular approaches include Bayesian Optimization (BO, \cite{bayesopt_nando, bo_jmlr, nguyen2019knowing}) and Evolutionary Algorithms (EAs, \cite{evolution2, evolution3}). Using automated methods for RL (AutoRL) is crucial for accessibility and for generalization, since different environments typically require totally different hyperparameters \cite{deeprlmatters}. Furthermore, it may even be possible to improve performance of existing methods using learned parameters. In fact, BO was revealed to play a valuable role in AlphaGo, improving the win percentage before the final match with Lee Sedol \cite{bo_alphago}.

A particularly promising approach, Population Based Training (PBT, \cite{pbt, pbt2}), showed it is possible to achieve impressive performance by updating both weights and hyperparameters during a \emph{single} training run of a population of agents. PBT works in a similar fashion to a human observing experiments, periodically replacing weaker performers with superior ones. PBT has shown to be particularly effective in reinforcement learning, and has been used in a series of recent works to improve performance \cite{kickstarting, liu2018emergent, impala, Jaderberg859}. 

\begin{figure}[h]
    %\begin{minipage}{0.99\textwidth}
    \centering\subfigure{\includegraphics[width=.8\linewidth]{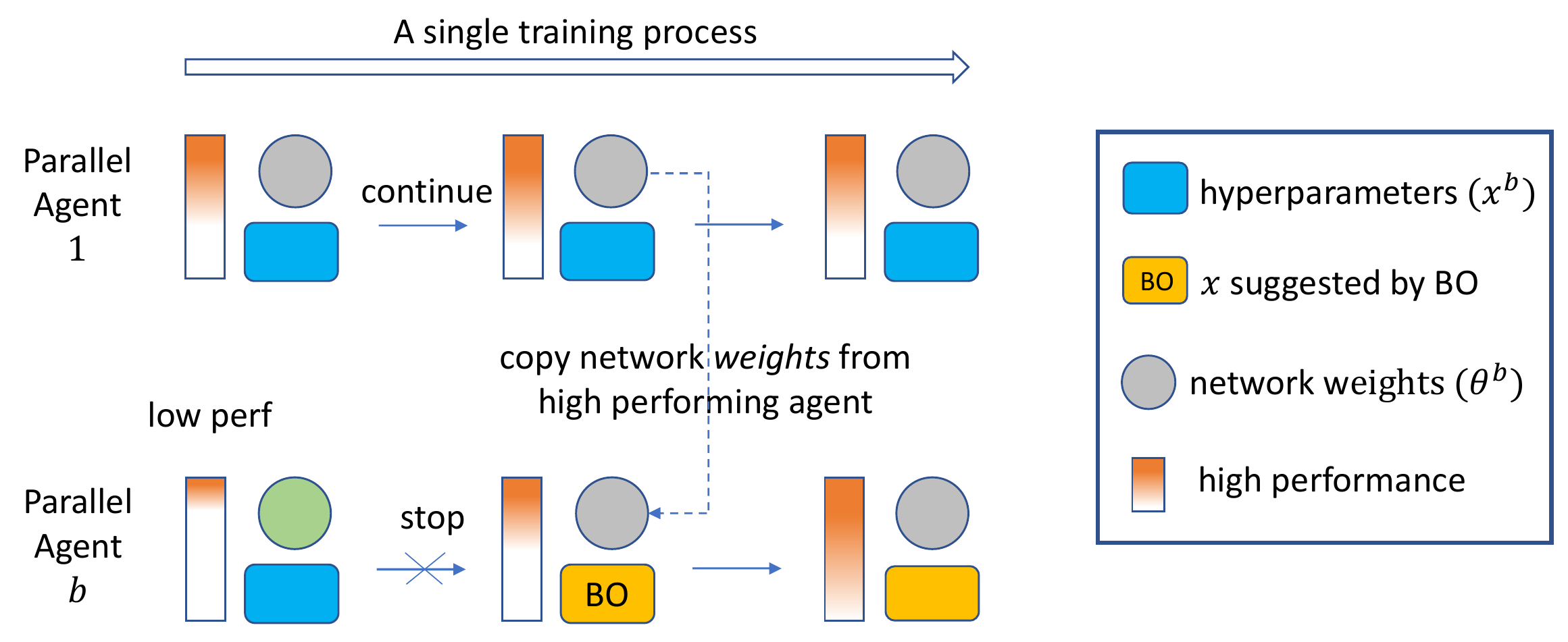}}
    \vspace{-1mm}
    \caption{\footnotesize{Population-Based Bandit Optimization: a population of agents is trained in parallel. Each agent has weights (grey) and hyperparameters (blue). The agents are evaluated periodically (orange bar), and if an agent is underperforming, it's weights are replaced by randomly copying one of the better performing agents, and its hyperparameters are selected using Bayesian Optimization.}}
    \label{figure:demo}
    \vspace{-3mm}
    %\end{minipage}
\end{figure}

However, PBT's Achilles heel comes from its reliance on heuristics for exploring the hyperparameter space. This essentially creates new meta-parameters, which need to be tuned. In many cases, PBT underperforms a random baseline without vast computational resources, since small populations can collapse to a suboptimal mode. In addition, PBT lacks theoretical grounding, given that greedy exploration methods suffer from unbounded regret.

Our key contribution is the first provably efficient PBT-style algorithm, Population-Based Bandit Optimization, or PB2 (Fig. \ref{figure:demo}). To do this, we draw the connection between maximizing the reward in a PBT-style training procedure to minimizing the regret in bandit optimization \cite{gp_ucb, gp_bucb}. Formally, we consider the online hyperparameter selection problem as batch Gaussian Process bandit optimization of a time-varying function. PB2 is more computationally efficient than the existing batch BO approaches \cite{desautels2014parallelizing,gonzalez2016batch} since it can (1) learn the optimal schedule of hyperparameters and (2) optimize them in a single training run. We derive a bound on the cumulative regret for PB2, the first such result for a PBT-style algorithm. Furthermore, we show in a series of RL experiments that PB2 is able to achieve high rewards with a modest computational budget. %when training with a small population size, where PBT often fails.

% The rest of the paper is organized as follows:
% \begin{itemize}
%     \item In Section \ref{sec:background} we formalize the online hyperparameter optimization problem.
%     \item In Section \ref{sec:pb2} we introduce the Population-Based Bandits Algorithm, and present our theoretical results.
%     \item In Section \ref{sec:related} we review related work.
%     \item In Section \ref{sec:experiments} we present our experimental results.
% \end{itemize}

% Finally, we conclude in Section \ref{sec:conclusion}, before discussing the broader impact of our work. 

%% file: preliminaries.tex
\section{Problem Statement} \label{sec:background}

In this paper, we consider the problem of selecting optimal hyperparameters, $x_t^b$, from a compact, convex subset $\mathcal{D}\in \mathbb{R}^d$ where $d$ is the number of hyperparameters. Here the index $b$ refers to the $b$th agent in a population/batch, and the subscript $t$ represents the number of timesteps/epochs/iterations elapsed during the training of a neural network.
Particularly, we consider the \textit{schedule} of optimal hyperparameters over time $\left( x_t^b \right)_{t=1,...T}$.

Population-Based Training (PBT, \cite{pbt}) is a well-known algorithm for learning a schedule of hyperparameters by training a population (or batch) of $B$ agents in parallel. Each agent $b\in B$ has both hyperparameters $x_t^b \in \mathbb{R}^d$ and weights $\theta_t^b$. At every $t_{\mathrm{ready}}$ step interval (i.e. if $t \mod t_{\mathrm{ready}} = 0$), the agents are ranked and the worst performing agents are replaced with members of the best performing agents ($A \subset B$) as follows:
\begin{itemize}
    \item \textbf{Weights} ($\theta_t^b$): copied from one of the best performing agents, i.e. $\theta_t^j \sim \mathrm{Unif}\{\theta_t^j\}_{j \in A}$.
    \item \textbf{Hyperparameters} ($x_t^b$): with probability $\epsilon$ it uses \textbf{random exploration}, and re-samples from the original distribution, otherwise, it uses \textbf{greedy exploration}, and perturbs one of the best performing agents, i.e $\{x^j \times \lambda\}_{j\in A}, \lambda \sim [0.8, 1.2]$. 
\end{itemize}
This leads to the learning of hyperparameter schedules -- a single agent can have different configurations at different time-steps during a single training process. This is important in the context of training deep reinforcement learning agents as dynamic hyperparameter schedules have been shown to be effective \cite{hoof, pbt, impala}. On the other hand, most of the existing hyperparameter optimization approaches aim to find a fixed set of hyperparameters over the course of optimization.
%\paragraph{Why Minimize Regret?} 

To formalize this problem, let $F_t(x_t)$ be an objective function under a given set of hyperparameters at timestep $t$. An example of $F_t(x_t)$ could be the reward for a deep RL agent. When training for a total of $T$ steps, our goal is to maximize the final performance $F_T(x_T)$. We formulate this problem as optimizing the time-varying black-box reward function $f_t$, over $\mathcal{D}$. Every $t_{\mathrm{ready}}$ steps, we observe and record noisy observations, $y_t = f_t(x_t) + \epsilon_t$, where $\epsilon_t \sim \mathcal{N}(0, \sigma^2\mathbf{I})$ for some fixed $\sigma^2$. The function $f_t$ represents the change in $F_t$ after training for $t_{\mathrm{ready}}$ steps, i.e. $F_t - F_{t-t_\mathrm{ready}}$. We define the best choice at each timestep as $x_t^* = \argmax_{x_t\in\mathcal{D}}f_t(x_t)$, and so the \textbf{regret} of each decision as $r_t = f_t(x_t^*) - f_t(x_t)$. 

\begin{lemma} \label{lem:maxreward_minregret}
Maximizing the final performance $F_T$ of a model with respect to a given hyperparameter schedule $\{x_t\}_{t=1}^T$ is equivalent to maximizing the time-varying black-box function $f_t(x_t)$ and minimizing the corresponding cumulative regret $r_t(x_t)$,
\begin{align}
\max F_T(x_T)= \max \sum^T_{t=1} f_t(x_t) =  \min \sum^T_{t=1} r_t(x_t).
\end{align}
\end{lemma}
%We refer to Sec. \ref{} in the appendix for the proof.
In subsequent sections, we present a time-varying bandit approach which is used to minimize the cumulative regret $R_T = \sum_{t=1}^T r_t$. Lemma \ref{lem:maxreward_minregret} shows this is equivalent to maximizing the final performance/reward of a neural network model (see: Section \ref{sec:theory} in the appendix for the proof).

%% file: algs.tex
\section{Population-Based Bandit Optimization}
\label{sec:pb2}

We now introduce Population-Based Bandit Optimization (PB2) for optimizing the hyperparameter schedule $\left(x_t^b \right),\forall t={1,...,T}$ using parallel agents $b={1,...,B}$. After each agent $b$ completes $t_\mathrm{ready}$ training steps, we store the data $(y_t^b, t, x_t^b)$ in a dataset $D_t$ which will be used to make an informed decision for the next set of hyperparameters.
%When it comes to selecting new hyperparamters, $x_{t+1}$, we can use the existing data $D_t$ to model the function $f_t$ with a GP. 

% The core structure of the PB2 algorithm inherits the nice property of PBT in training a population of agents in parallel and copying \textit{weights} from superior models. the key difference comes in the selection of $x_t^b$. 

We below present the mechanism to select the next hyperparameters for parallel agents. Motivated by the equivalence of the maximized reward and minimized cumulative regret in Lemma \ref{lem:maxreward_minregret}, we propose the parallel time-varying bandit optimization.

\subsection{Parallel Gaussian Process Bandits for a Time-Varying Function}\label{sec:TVGP}

We first describe the time-varying Gaussian process as the surrogate models, then we extend it to the parallel setting for our PB2 algorithm.
%\cite{gp_ucb} were the first to derive a sublinear regret bound for GP-bandit optimization, using the following acquisition function:
% \begin{equation}
%     x_t = \argmax_{x\in \mathcal{D}} \mu_{t-1}(x) + \sqrt{\beta_t} \sigma_{t-1}(x)
% \end{equation}

%However, if we wish to optimize $f_t$, then we need to consider the fact that the function is \texit{changing through time}. This is crucial, as one would expect a neural network to respond differently to  hyperparameters at different stages of training.

% \cite{gp_ucb} present the GP optimization under the bandit setting. They derive the principle way to adapt the decision function for achieving convergence guarantee. For handling our neural network behavior time-varying bandit --

%\vcom{I will write it here}

\paragraph{Time-varying Gaussian process as the surrogate model.}
Following previous works in the GP-bandit literature \cite{gp_ucb}, we model $f_t$ using a Gaussian Process (GP, \cite{RasmussenGP}) which is specified by a mean function $\mu_t : \mathcal{X} \rightarrow \mathbb{R}$ and a kernel (covariance function) $k : \mathcal{D} \times \mathcal{D}  \rightarrow \mathbb{R}$. If $f_t \sim  GP(\mu_t, k)$, then $f_t(x_t)$ is distributed normally $\mathcal{N} (\mu_t(x_t), k(x_t, x_t)) \text{ for all } x_t \in \mathcal{D}$. After we have observed $T$ data points $\{(x_t, f(x_t))\}_{t=1}^T$, the GP posterior belief at new point $x'_t \in \mathcal{D}$, $f_t(x'_t)$ follows  a Gaussian distribution with mean $\mu_t(x')$ and variance $\sigma_t^2(x')$ as: 
% \begin{minipage}[t]{0.48\columnwidth}
% \begin{align}
%  \mu\left(\bx'\right)	=\mathbf{k_{*X}} \left[ \mathbf{K} + \sigma^2_n \idenmat \right]^{-1}\mathbf{y}
%     \label{eq:GPmean}
%     \end{align}
% \end{minipage}%
% \begin{minipage}[t]{0.48\columnwidth}%
% \begin{align}
%     \sigma^{2}\left(\bx_*\right)	=k_{**}-\mathbf{k_{*X}}\left[ \mathbf{K} + \sigma^2_n \idenmat \right]^{-1}\mathbf{k'_{*X}} \label{eq:GPvar}
% \end{align}
% \end{minipage}
\begin{equation}
\label{eqn:mu}
    \mu_t(x') \coloneqq \textbf{k}_t(x')^T(\textbf{K}_t + \sigma^2\textbf{I})^{-1}\mathbf{y}_t
\end{equation}
\begin{equation}
\label{eqn:sig}
    \sigma_t^2(x') \coloneqq k(x', x') - \textbf{k}_t(x')^T(\textbf{K}_t +\sigma^2\textbf{I})^{-1}\textbf{k}_t(x'),
\end{equation}
where $\mathbf{K}_t \coloneqq \{k(x_i, x_j)\}_{i, j=1}^t$ and $\mathbf{k}_t \coloneqq \{k(x_i, x'_t)\}_{i=1}^t$. The GP predictive mean and variance above will later be used to represent the exploration-exploitation trade-off in making decision under the presence of uncertainty. %Equipped with (\ref{eqn:mu}) and (\ref{eqn:sig}), we can select new samples by maximizing an \textit{acquisition function} described in Eq. (\ref{eq:gp_bucb_acq}). 

To represent the non-stationary nature of a neural network training process, we cast the problem of optimizing neural network hyperparameters as time-varying bandit optimization. We follow \cite{bogunovic2016time} to formulate this problem by modeling the reward function under the time-varying setting as follows:
\begin{equation}
     f_1(x) = g_1(x), \; \; f_{t+1}(x) = \sqrt{1-\omega}f_t(x) + \sqrt{\omega}fg_{t+1}(x) \; \: \: \:  \forall t \geq 2,
\end{equation}
where $g_1, g_2, ...$ are independent random functions with $g\sim GP(0, k)$ and $\omega \in[0,1]$ models how the function varies with time, such that if $\omega=0$ we return to GP-UCB and if $\omega=1$ then each evaluation is completely independent. This model introduces a new hyperparameter ($\omega$), however, we note it can be optimized by maximizing the marginal likelihood for a trivial additional cost compared to the expensive function evaluations (we include additional details on this in the Appendix). This leads to the extensions of Eqs. (\ref{eqn:mu}) and (\ref{eqn:sig}) 
using the new covariance matrix $\Tilde{\textbf{K}}_t = \textbf{K}_t \circ \textbf{K}_t^{\mathrm{time}}$ where $\textbf{K}_t^{\mathrm{time}} = [(1-\omega)^{|i-j|/2}]_{i,j=1}^T$ and $\Tilde{\textbf{k}}_t(x) = \textbf{k}_t \circ \textbf{k}_t^{time}$ with $\textbf{k}_t^{time} = [(1-\omega)^{(T+1-i)/2}]_{i=1}^T$. Here $\circ$ refers to the Hadamard product. 

%We will consider translation invariant kernels $k$ such as the squared exponential (SE) and Matern (see: Section \ref{sec:theory}). 

% This leads to extensions to (\ref{eqn:mu}) and (\ref{eqn:sig}) as follows:
% \begin{equation}
% \label{eqn:tv_mu}
%     \Tilde{\mu}_t(x') \coloneqq \Tilde{\textbf{k}}_t(x')^T(\Tilde{\textbf{K}}_t + \sigma^2\textbf{I})^{-1}\mathbf{y}_t
% \end{equation}
% \begin{equation}
% \label{eqn:tv_sig}
%     \Tilde{\sigma}_t^2(x') \coloneqq k(x', x') - \Tilde{\textbf{k}}_t(x')^T(\Tilde{\textbf{K}}_t +\sigma^2\textbf{I})^{-1}\Tilde{\textbf{k}}_t(x')
% \end{equation}

%This essentially equates to forgetting in a ``smooth'' fashion. 
%Interestingly, despite the additional complexity of a time-varying function, \cite{bogunovic2016time} show this algorithm also achieves sublinear regret: $R_T  \leq \sqrt{C_1 T\beta_T\Tilde{\gamma}_T}+2$ with probability $1-\delta$, where $C_1 = 8/\text{log}(1+\sigma^{-2})$, and $\beta_t$ is chosen appropriately. To see more detail on how this bound is achieved, please see the Appendix, Section \ref{sec:theory} or \cite{bogunovic2016time}. 

% \vcom{We will present the derivative over $\omega$ and put in the supplement for completeness. We may also discuss on how to optimize the block length $\Tilde{N}$. Instead of cross-validation, can we replace it by "optimization". It is basically the same thing and will lead to the same result.}

\paragraph{Selecting hyperparameters for parallel agents.}

In the PBT setting, we consider an entire population of parallel agents. This changes the problem from a sequential to a \textit{batch} blackbox optimization. This poses an additional challenge to select multiple points simultaneously $x^b_t$ without full knowledge of all $\{(x^j_t, y^j_t)\}_{j=1}^{j-1}$. A key observation in \cite{gp_bucb} is that since a GP's variance (Eqn. \ref{eqn:sig}) does not depend on $y_t$, the acquisition function can  account for incomplete trials by updating the uncertainty at the pending evaluation points. Concretely, we define $x^b_{t}$ to be the $b$-th point selected in a batch, after $t$ timesteps. This point may draw on information from $t + (b-1)$ previously selected points. In the single agent, sequential case, we set $B=1$ and recover $t,b = t-1$. Thus, at the iteration $t$, we find a next batch of $B$ samples $\left[x^1_t,x^2_t,...x^B_t \right]$  by sequentially maximizing the following acquisition function:
\begin{equation}
\label{eq:gp_bucb_acq}
    x^b_t = \argmax_{x\in \mathcal{D}} \mu_{{t,1}}(x) + \sqrt{\beta_t} \sigma_{{t,b}}(x),\forall b=1,...B
\end{equation}
for $\beta_t >0$. In Eqn. (\ref{eq:gp_bucb_acq}) we have the mean from the previous batch ($\mu_{t,1}(x)$) which is fixed, but can update the uncertainty using our knowledge of the agents currently training ($\sigma_{t,b}(x)$). This significantly reduces redundancy, as the model is able to explore distinct regions of the space. %while it still preserves to find the high function value encoded in $\mu_{{t,1}}(x)$.

\subsection{PB2 Algorithm and Convergence Guarantee}

To estimate the GP predictive distribution in Eqs. (\ref{eqn:mu}, \ref{eqn:sig}), we use the product form  $\tilde{k} = k^{\mathrm{SE}} \circ k^{\mathrm{time}}$ \cite{krause2011contextual,bogunovic2016time}, which considers the time varying nature of training a neural network. Then, we use Eqn. (\ref{eq:gp_bucb_acq}) to select a batch of points by utilizing the reduction in uncertainty for models currently training. As far as we know, this is the first use of a time-varying kernel in the PBT-style setting. Unlike PBT, this allows us to efficiently make use of data from previous trials when selecting new configurations, rather than reverting to a uniform prior, or perturbing existing configurations. The full algorithm is shown in Algorithm \ref{Alg:pb2}.

\scalebox{0.95}{
\begin{minipage}{.99\linewidth}
    \begin{algorithm}[H]
    \textbf{Initialize:} Network weights $\{\theta_0^b\}_{b=1}^B$, hyperparameters $\{x_0^b\}_{b=1}^B$, dataset $D_0 = \emptyset$ \\
    \textbf{(in parallel)} \For{$t=1, \ldots, T-1$}{
      1. \textbf{Update Models:} $\theta_{t}^b \leftarrow \mathrm{step}(\theta_{t-1}^b | x_{t-1}^b) $ \\
      2. \textbf{Evaluate Models:} $y_t^b = F_t(x_{t}^b)- F_{t-1}(x_{t-1}^b) + \epsilon_t$ for all $b$ \\
      3. \textbf{Record Data:} $D_t = D_{t-1} \cup \{(y_t^b, t, x_t^b)\}_{b=1}^B$ \\
      4. If $t \mod t_{\mathrm{ready}} = 0$: 
      \begin{itemize}
          \item \textbf{Copy weights:} Rank agents, if $\theta^b$ is in the bottom $\lambda\%$ then copy weights $\theta^j$ from the top $\lambda\%$.
          \item \textbf{Select hyperparameters:} Fit a GP model to $D_t$ and select hyper-parameters $x^b_t, \forall b \le B$ by maximizing Eq. (\ref{eq:gp_bucb_acq}).
      \end{itemize}}
     \textbf{Return the best trained model $\theta$}
     \caption{Population-Based Bandit Optimization (PB2)}
    \label{Alg:pb2}
    \end{algorithm}
\end{minipage}
}

\noindent Next we present our main theoretical result, showing that PB2 is able to achieve sublinear regret when the time-varying function is correlated. This is the first such result for a PBT-style algorithm. 
\begin{theorem}
Let the domain $\mathcal{D}\subset[0,r]^{d}$ be compact and convex where $d$ is the dimension and suppose that the kernel is such that $f_t \sim GP(0,k)$ is almost surely continuously differentiable and satisfies Lipschitz assumptions $\forall L_t \ge0,t\le\mathcal{T},p(\sup\left|\frac{\partial f_{t}(\bx)}{\partial\bx^{(d)}}\right|\ge L_t)\le ae^{-\left(L_t/b\right)^{2}}$ for some $a,b$. Pick $\delta\in(0,1)$, set $\beta_{T}=2\log\frac{\pi^{2}T^{2}}{2\delta}+2d\log rdbT^{2}\sqrt{\log\frac{da\pi^{2}T^{2}}{2\delta}}$ and define $C_{1}=32/\log(1+\sigma_{f}^{2})$, the PB2 algorithm satisfies the following regret bound after $T$ time steps over $B$ parallel agents with probability at least $1-\delta$:
\begin{align*}
R_{TB}=\sum_{t=1}^{T}f_{t}(\bx_{t}^{*})-\max_{b=1,...,B} f_{t}(\bx_{t,b})\le \sqrt{C_{1}T\beta_{T}\left(\frac{T}{\tilde{N}B}+1\right)\left(\gamma_{\tilde{N}B}+\left[\tilde{N}B\right]^{3}\omega\right)}+2
\end{align*}
the bound holds for any block length
$\tilde{N}\in\left\{ 1,...,T\right\} $ and $B\ll T$.
\end{theorem}

This result shows the regret for PB2 decreases as we increase the population size $B$. This should be expected, since adding computational resources should benefit final performance (which is equivalent to minimizing regret). We demonstrate this property in Table \ref{table:rl_results}. When using a single agent $B=1$, our bound becomes the time-varying GP-UCB setting in \cite{bogunovic2016time}. 

In our setting, if the time-varying function is highly correlated, i.e., the information between $f_1(.)$ and $f_T(.)$ does not change significantly, we have  $\omega \rightarrow 0$ and $\tilde{N} \rightarrow T$. Then, the regret bound grows sublinearly with the number of iterations $T$, i.e., $\lim_{T \rightarrow\infty}\frac{R_{TB}}{T}=0$. On the other hand (in the worst case), if the time-varying function is not correlated at all, such as $\tilde{N} \rightarrow 1$ and $ \omega \rightarrow 1$, then PB2 achieves linear regret \citep{bogunovic2016time}. 

\paragraph{Remark.}
This theoretical result is novel and significant in two folds. First, by showing the equivalent between maximizing reward and minimizing the bandit regret, this regret bound quantifies the gap between the optimal reward and the achieved reward (using parameters selected by PB2). The regret bound extends the result established by \cite{gp_ucb,bogunovic2016time}, to a more general case with parallelization. To the best of our knowledge, this is the first kind of convergence guarantee in a PBT-style algorithm. 

Our approach is advantageous against all existing batch BO approaches \cite{gp_bucb,gonzalez2016batch} in that PB2 considers optimization in a \textit{single training run} of parallel agents while the existing works need to evaluate using \textit{multiple training runs} of parallel agents which are more expensive. In addition, PB2 can learn a \textit{schedule} of hyperparameters while the existing batch BO can only learn a static configuration. 

%Note that in practice, we use the same fixed \vcom{they used $0.8 \log ( 4\times t)$ so it is not fixed } parameterization for $\beta_t$ as in \cite{bogunovic2016time} \vcom{this sentence could be placed in the experimental setting.}.

%% file: related_work.tex
\section{Related Work}
\label{sec:related}

\textbf{Hyperparameter Optimization.} Hyperparameter optimization \cite{hpo_datasets, hyperparam_opt, hpo_hutter} is a crucial component of any high performing machine learning model. In the past, methods such as grid search and random search \cite{randomsearch} have proved popular, however, with increased focus on Automated Machine Learning (AutoML, \cite{automl_book}), there has been a great deal of progress moving beyond these approaches. This success has led to a variety of tools becoming available \cite{hutter_sklearn, autostatistician}. 

\textbf{Population Based Approaches.} Taking inspiration from biology, population-based methods have proved effective for blackbox optimization problems \cite{cmaes}. Evolutionary Algorithms (EAs, \cite{evo_vs_part}) take many forms, one such method is Lamarckian EAs \cite{lamarkian} in which parameters are inherited whilst hyperparameters are evolved. Meanwhile, other methods learn both hyperparameters and weights \cite{evolution2, evolution3}. These works motivate the recently introduced PBT algorithm \cite{pbt}, whereby network parameters are learned via gradient descent, while hyperparameters are evolved. Its key strengths lie in the ability to learn high performing hyperparameter schedules in a single training run, leading to strong performance in a variety of settings \cite{liu2018emergent, pbt2, impala, kickstarting}. PBT works by focusing on high performing configurations, however, this greedy property means it is unlikely to explore areas of the search space which are ``late bloomers''. In addition, the core components of the algorithm rely on handcrafted meta-hyperparameters, such as the degree of mutation. These design choices mean the algorithm lacks theoretical guarantees. We address these issues in our work. 

\textbf{Bayesian Optimization.} Bayesian Optimization (BO \cite{bayesopt_nando, shahriari2015taking,CoCaBO}) is a sequential model-based blackbox optimization method \cite{sequentialmb}. BO works by building a surrogate model of the blackbox function, typically taken to be a Gaussian Process \cite{RasmussenGP}. BO has been shown to produce state-of-the-art results in terms of sample efficiency, making dramatic gains in several prominent use cases \cite{bo_alphago}. Over the past few years there has been increasing focus on distributed implementations \cite{ bayeslocalpen, async_alvi} which seek to select a batch of configurations for concurrent evaluation. Despite this increased efficiency, these methods still require multiple training runs. Another recent method, Freeze-Thaw Bayesian Optimization \cite{freezethaw} considers a `bag' of current solutions, with their future loss assumed to follow an exponential decay. The next model to optimize is chosen via entropy maximisation. This approach bears similarities in principle with PBT, but from a Bayesian perspective. The method, however, makes  assumptions about the shape of the loss curve, does not adapt hyperparameters during optimization (as in PBT) and is sequential rather than parallelized. Our approach takes appealing properties of both these methods, using the Bayesian principles but adapting in an online fashion.

\textbf{Hybrid Algorithms.} Bayesian and Evolutionary approaches have been combined in the past. In \cite{boagecco}, the authors demonstrate using an early version of BO improves a simple Genetic Algorithm, while \cite{harmless} shows promising results through combining BO and random search. In addition, the recently popular Hyperband algorithm \cite{hyperband}, was shown to exhibit stronger performance with a BO controller \cite{bohb}. The main weakness of these methods is their inability to learn schedules.  

\textbf{Hyperparameter Optimization for Reinforcement Learning (AutoRL).} Finally, methods have been proposed specifically for RL, dating back to the early 1990s \cite{Sutton92}. These methods have typically had a narrower focus, optimizing an individual parameter. More recent work \cite{hoof} proposes to adapt hyperparameters online, by exploring different configurations in an off-policy manner in between iterations. Additionally \cite{ompac} concurrently proposed a similar to PBT, using evolutionary hyperparameter updates specifically for RL. These methods show the benefit of re-using samples, and we believe it would be interesting future work to consider augmenting our method with off-policy or synthetic samples (from a learned dynamics model) for the specific RL use case.

%% file: experiments.tex
\section{Experiments}
\label{sec:experiments}

We focus our experiments on the RL setting, since it is notoriously sensitive to hyperparameters \cite{deeprlmatters}. We evaluate population sizes of $B\in\{4,8\}$, which means the algorithm can be run locally on most modern computers. This is significantly less than the $B>20$ used in the original PBT paper \cite{pbt}. 

All experiments were conducted using the tune library \cite{liaw2018tune, rllib}\footnote{See code here: \url{https://github.com/jparkerholder/PB2}.}. Our GP implementation is extended from GPy \cite{gpy2014}, where we use the squared exponential kernel in combination with the time kernel as described in Section \ref{sec:pb2}. We optimize all GP-kernel hyperparameters, as well as the block length $\Tilde{N}$, by maximizing the marginal likelihood \cite{RasmussenGP}. 

\subsection{On Policy Reinforcement Learning}

We consider optimizing a policy for continuous control problems from the OpenAI Gym \cite{gym}. In particular, we seek to optimize the hyperparameters for Proximal Policy Optimization (PPO, \cite{schulman2017proximal}), for the following tasks: $\mathrm{BipedalWalker}$, $\mathrm{LunarLanderContinuous}$, $\mathrm{Hopper}$ and $\mathrm{InvertedDoublePendulum}$. 

Our primary benchmark is PBT, where we use an identical configuration to PB2 aside from the selection of $x_t^b$ (the explore step). We also compare against a random search (RS) baseline \cite{randomsearch}. Random search is a challenging baseline because it does not make assumptions about the underlying problem, and typically achieves close to optimal performance asymptotically \cite{automl_book}. We include a Bayesian Optimization (BO) baseline which uses the Expected Improvement acquisition function. Finally, we also compare our results against a recent state-of-the-art distributed algorithm (ASHA, \cite{asha}). ASHA is one of the most recent distributed methods, shown to outperform PBT  for supervised learning. We believe we are the first to test it for RL.

\begin{table}[h]
    \centering
    \caption{\small{Median best performing agent across $10$ seeds. The best performing algorithms are bolded.}}
    \scalebox{0.9}{
    \begin{tabular}{ccccccc | c }
    \toprule
    & $B$ & RS & BO & ASHA & PBT & PB2 &  vs. PBT \\
    \midrule
    $\mathrm{BipedalWalker}$ & 4 &  234 & 133 & 236 & 223 & \textbf{276} & +24\% \\
    $\mathrm{LunarLanderContinuous}$ & 4 &  161 & 206 & 213 & 159 & \textbf{235} & +48\% \\
    $\mathrm{Hopper}$ & 4 & 1638 & 1760 & 1819 & 1492 & \textbf{2346} & +57\% \\
    $\mathrm{InvertedDoublePendulum}$ & 4 & 8094 & 8607 & 7899 & \textbf{8893} & 8179 & -8\% \\
    \midrule
    $\mathrm{BipedalWalker}$ & 8 & 240 & 237 & 255 & 277 & \textbf{291} & +5\% \\
    $\mathrm{LunarLanderContinuous}$ & 8 &  175 & 240 & 231 & 247 & \textbf{275} & +11\% \\    \bottomrule
\end{tabular}}
\label{table:rl_results}
\end{table}

\begin{figure}[h!]
\vspace{-3mm}
    \begin{minipage}{1\textwidth}
    \centering \subfigure[\textbf{BipedalWalker (4)}]{\includegraphics[width=.3\linewidth]{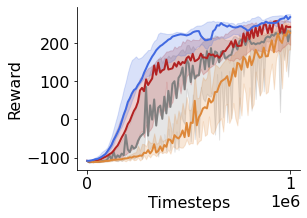}}
    \centering \subfigure[\small\textbf{LLC (4)}]{\includegraphics[width=.3\linewidth]{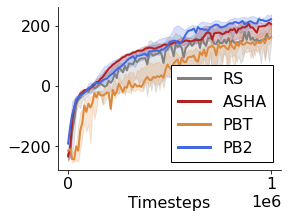}}  
    \centering \subfigure[\textbf{Hopper (4)}]{\includegraphics[width=.3\linewidth]{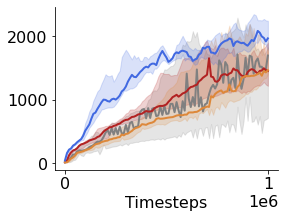}}  
    \centering \subfigure[\textbf{IDP (4)}]{\includegraphics[width=.3\linewidth]{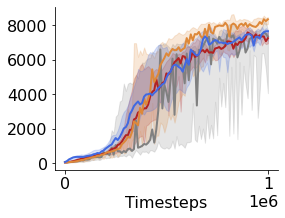}}
    \centering \subfigure[\textbf{BipedalWalker (8)}]{\includegraphics[width=.3\linewidth]{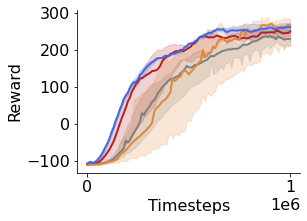}}
    \centering \subfigure[\small\textbf{LLC (8)}]{\includegraphics[width=.3\linewidth]{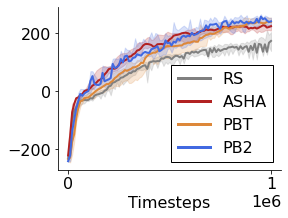}}  
    \caption{\small{Median best performing agent across ten seeds, with IQR shaded. Population size ($B$) is shown in brackets.}}
    \label{figure:rl4}
    \end{minipage}
    \vspace{-8mm}
\end{figure}

\begin{comment}
    \midrule
    $\mathrm{Hopper}$ & 8 &  2075 & 2336 & 2724 & \textbf{3227} \\
    $\mathrm{InvertedDoublePendulum}$ & 8 & 7943 & 8210 & 9320 & \textbf{9330} \\
\end{comment}

%  \cite{gae}

For all environments we use a neural network policy with two 32-unit hidden layers and $\mathrm{tanh}$ activations. During training, we optimize the following hyperparameters: batch size, learning rate, GAE parameter ($\lambda$, \cite{gae}) and PPO clip parameter ($\epsilon$). We use the same fixed ranges across all four environments (included in the Appendix Section \ref{appendix:details}). All experiments are conducted for $10^6$ environment timesteps, with the $t_\mathrm{ready}$ command triggered every $5\times10^4$ timesteps. For BO, we train each agent sequentially for $500k$ steps, and selects the best to train for the remaining budget. For ASHA, we initialize a population of $18$ agents to compare against $B=4$ and $48$ agents for $B=8$. These were chosen to achieve the same total budget with the grace period equal to the $t_\mathrm{ready}$ criteria for PBT and PB2. Given the typically noisy evaluation of policy gradient algorithms \cite{deeprlmatters} we repeat each experiment with ten seeds. We show the median best reward achieved from each run in Table \ref{table:rl_results} and plot the median best performing agent from each run, with the interquartile ranges (IQRs) in Fig. \ref{figure:rl4}. 

In almost all cases we see performance gains from PB2 vs. PBT. In fact, $3/4$ of cases PBT actually underperforms random search with the smaller population size ($B=4$), demonstrating its reliance on large computational resources. This is confirmed in the original PBT paper, where the smallest population size fails to outperform (see Table 1, \cite{pbt}). One possible explanation for this is the greediness of PBT leads to prematurely abandoning promising regions of the search space. Another is that the small changes in parameters (multiple of $0.8$ or $1.2$) is mis-specified for discovering the optimal regions, thus requiring more initial samples to sufficiently span the space. This may also be the case if there is a shift later in the optimization process. Interestingly, PBT does perform well for $\mathrm{InvertedDoublePendulum}$, this may be explained by the relative simplicity of the problem. We see that BO also performs well here, confirming that it may not require the same degree of adaptation during training as the other tasks (such as $\mathrm{BipedalWalker}$). 

For the larger population size we confirm the effectiveness of PBT, which outperforms ASHA and Random Search. However PB2 outperforms PBT by $+5\%$ and $+11\%$ for the $\mathrm{BipedalWalker}$ and $\mathrm{LunarLanderContinuous}$ environments respectively. This shows that PB2 is able to scale to larger computational budgets. 

Interestingly, the state-of-the-art supervised learning performance of ASHA fails to translate to RL, where it clearly performs worse than both PBT and PB2 for the larger setting, and performs worse than PB2 for the smaller one.

\begin{wrapfigure}{r}{0.33\textwidth}
        \vspace{-6mm}
        \subfigure{\includegraphics[width=0.99\linewidth]{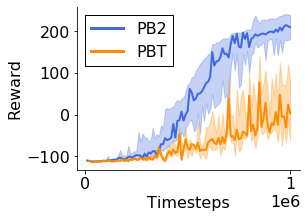}} 
        \vspace{-5mm}
        \caption{\footnotesize{Median curves, IQR shaded.}}
        \label{figure:batch}
\end{wrapfigure}

\textbf{Robustness to Hyperparameter Ranges} One key weakness of PBT is a reliance on a large population size to explore the hyperparameter space. This problem can be magnified if the hyperparameter range is mis-specified or unknown (in a sense, the bounds placed on the hyperparameters may require tuning). PB2 avoids this issue, since it is able to select a point anywhere in the range, so does not rely on random sampling or gradual movements to get to optimal regions. We evaluate this by re-running the $\mathrm{BipedalWalker}$ task with a batch size drawn from $\{5000, 200000\}$. This means many agents are initialized in a very inefficient way, since when an agent has a batch size of $200,000$ it is using $20\%$ of total training samples for a single gradient step. As we see in Fig. \ref{figure:batch}  the performance for PBT is significantly reduced, while PB2 is still able to learn good policies. While both methods perform worse than in Table \ref{table:rl_results}, PB2 still achieves a median best of $203$, vs. $-12$ for PBT. This is a critical issue, since for new problems we will not know the optimal hyperparameter ranges ex-ante, and thus require a method which is capable of learning without this knowledge.

\begin{wrapfigure}{r}{0.35\textwidth}
        \vspace{-4mm}
        \subfigure{\includegraphics[width=0.85\linewidth]{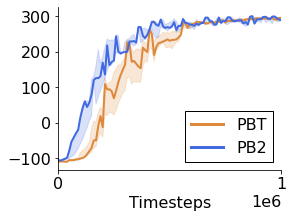}} 
        \vspace{-5mm}
        \caption{\footnotesize{Median curves, IQR shaded.}}
        \vspace{-5mm}
        \label{fig:b16}
\end{wrapfigure}

\textbf{Scaling to Larger Populations.} For PB2 to be broadly useful, it needs the ability to scale when more resources are available. To test this we repeated the $\mathrm{BipedalWalker}$ experiment with $B=16$. As we see in Figure \ref{fig:b16}, both PBT and PB2 achieve optimal rewards ($>300$), but PB2 is still more efficient. This is a promising initial result, although of course it will be interesting to test even larger settings in the future.

\subsection{Off Policy Reinforcement Learning}

We now evaluate PB2 in a larger setting, optimizing hyperparameters for IMPALA \cite{impala}. in the $\mathrm{breakout}$ and $\mathrm{Space Invaders}$ environments from the Arcade Learning Environment \cite{ale}. In the original IMPALA paper, the best results come with the use of PBT with a population size of $B=24$. Here we optimize the same three hyperparameters as in the original paper, but with a much smaller population ($B=4$), making it essential to efficiently explore the hyperparameter space.  

We train for $10$ million timesteps, equivalent to $40$ million frames, and set $t_\mathrm{ready}$ to $5\times10^5$ timesteps. We repeat each experiment for $7$ random seeds. PB2 achieves performance which is comparable with the hand-tuned performance reported in the rllib  implementation,\footnote{See ``RLlib IMPALA 32-workers''  here: \url{https://github.com/ray-project/rl-experiments}} while PBT underperforms, particularly in the $\mathrm{Space Invaders}$ environment. 

\begin{figure}[h]
    \begin{minipage}{1\textwidth}
        \subfigure[Breakout]{\includegraphics[width=0.29\linewidth]{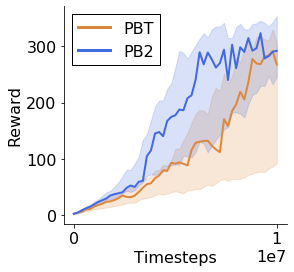}} 
        \subfigure{\includegraphics[width=0.19\linewidth]{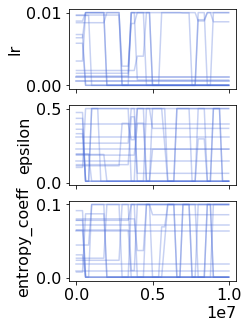}} 
        \subfigure[Space Invaders]{\includegraphics[width=0.29\linewidth]{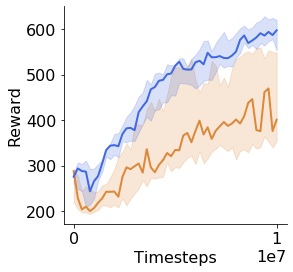}} 
        \subfigure{\includegraphics[width=0.19\linewidth]{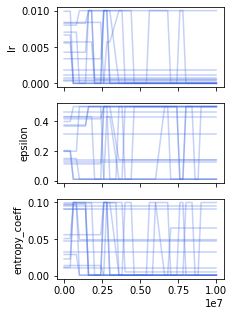}} 

    \caption{\small{In both (a) and (b) we show training performance on the left, with median curves for seven seeds and inter-quartile range shaded. On the right, we show all agent configurations found by PB2.}}
    \label{figure:atari}
    \end{minipage}
\end{figure}

For each environment in Fig. \ref{figure:atari} we include the hyperparameters used for all agents across all seeds of training. PB2 effectively explores the entire range of parameters, which enables it to find optimal configurations even with a small number of trials. More details are in the Appendix, Section \ref{appendix:details}.

%% file: conclusions.tex
\section{Conclusion}
\label{sec:conclusion}

We introduced Population-Based Bandits (PB2), the first PBT-style algorithm with sublinear regret guarantees. PB2 replaces the heuristics from the original PBT algorithm with theoretically guided GP-bandit optimization. This allows it to balance exploration and exploitation in a principled manner, preventing mode collapse to suboptimal regions of the hyperparameter space and making it possible to find high performing configurations with a small computational budget. Our algorithm complements the existing approaches for optimizing hyperparameters for deep learning frameworks. We believe the gains for reinforcement learning will be particularly useful, given the number of hyperparameters present, and the difficulty in optimizing them with existing techniques.  

Finally, we believe there are several future directions opened by taking our approach, such as updating population sizes based on the value of the acquisition function, and extending the search space to select the optimization algorithms or neural network architectures with BO. We also believe there may be further gains in reinforcement learning experiments through making use of off-policy data \cite{hoof}. We leave these to exciting future work.

%% file: appendix.tex
\section*{Appendix}

\begin{comment}

\section{Notation summary}
 Please refer to Table \ref{tab:Notation-list} for a description of the notations used in our paper.
\begin{table*}[t]
  \caption{Notation list} 
  \label{tab:Notation-list}
  \centering
    \begin{tabular}{lll}
    \toprule
    \textbf{Notation}    & \textbf{Type}     & \textbf{Meaning} \\\midrule
    $l$ & scalar & {lengthscale for RBF kernel}\\\midrule
 
    % $\mu(x)$, $\sigma(x)$ & function & GP predictive mean and variance\\\midrule
    
    $\mathcal{X}\in\mathbb{R}^{d}$ & search domain & continuous search space where $d$
is the dimension\\\midrule

    $d$ & scalar & dimension of the continuous variable\\\midrule
    
    $c$ & scalar & dimension of categorical variables\\\midrule
    
    $\mathbf{x}_{t}$ & vector & a continuous selection by BO at iteration $t$\\\midrule

    $\mathbf{h}_{t}=[h_{t,1},...,h_{t,c}]$ & vector & vector of categorical variables\\\midrule
 
    % $h_{t,c}$ & scalar &  \makecell[l]{selection of category $c$ which can take $N_{c}$ choices, \\ i.e., $h_{t,c} \in \{1,...,N_{c}\}$} \\\midrule

    $\mathcal{D}_{t}$ & set & observation set $D_{t}=\{z_{i},y_{i}\}_{i=1}^{t}$\\\midrule
 
    \end{tabular}
\end{table*}
\end{comment}

\input{additional_experiments.tex}

\input{appendix_experiments.tex}

%% file: additional_experiments.tex
\section{Additional Experiments}
\label{sec:additional_exps}

\textbf{Supervised Learning} While the primary motivation for our work is RL, we also evaluated PB2 the supervised learning case, to test the generality of the method. Concretely, we used PB2 to optimize six hyperparameters for a Convolutional Neural Network (CNN) on the CIFAR-10 dataset \citep{cifar10}. In each setting we randomly sample the initial hyperparameter configurations and train on half of the dataset for $50$ epochs. We use $B=4$ agents for RS, PBT and PB2, with $t_\mathrm{ready}$ as $5$ epochs. For ASHA we have the same maximum budget across all agents but begin with a population size of $16$.

\begin{table}[h]
    \centering
    \caption{\small{Median best performing agent across $5$ seeds. The best performing methods are bolded.}}
    \scalebox{0.9}{
    \begin{tabular}{l*{5}{c}r}
    \toprule
     & RS & ASHA & PBT & PB2  \\
    \midrule
    Test Accuracy & $84.43$ & $88.85$ & $87.20$ & $\mathbf{89.10}$ \\
    \bottomrule
\end{tabular}}
\end{table}
\label{table:sl_results}

In Table \ref{table:sl_results} we present the median best performing agent from five runs of each algorithm. We see that PB2 outperforms all other methods, including ASHA, which was specifically designed for SL problems \cite{asha}. This result indicates PB2 may be useful for a vast array of applications.

\begin{comment}

\subsection{Reinforcement Learning}
\label{sec:rlexp}

\begin{figure}[h!]
    \begin{minipage}{0.99\textwidth}

    \vspace{-5mm}
    \caption{\small{Median best performing agent across ten seeds for a population size $B=8$. IQR shaded.}}
    \label{figure:rl8}
    \end{minipage}
\end{figure}

\end{comment}

%% file: appendix_experiments.tex
\begin{comment}
\subsection{Exploration Profile}

In order to understand the exploration properties of $\mathrm{PB2}$, in Fig. \ref{fig:schedules} we show the learned batch size for $\mathrm{PB2}$  and $\mathrm{PBT}$ across three of the seeds of $\mathrm{InvertedDoublePendulum}$. We see that $\mathrm{PBT}$ very quickly collapses to a single mode, while $\mathrm{PB2}$  continues to explore the space throughout learning. We chose to focus on this hyperparameter since for many of the environments it appeared that smaller batch sizes led to faster training initially, yet during later stages a larger batch size may be beneficial to finetune performance. This is confirmed by the fact that for the same experiment the best performing $\mathrm{ASHA}$ agent also had a small batch size (median = $7000$), which possibly explained why it was never able to fine tune enough to achieve a high reward. 

\begin{figure}[h]
    \begin{minipage}{0.99\textwidth}
    \centering\subfigure[\textbf{PBT}]{\includegraphics[width=.24\linewidth]{plots/schedule_pbt_batch.png}} 
    \centering\subfigure[\textbf{PB2}]{\includegraphics[width=.24\linewidth]{plots/schedule_ppbt_batch.png}}
    \vspace{-3mm}
    \caption{\small{Train batch size during optimization. We see that $\mathrm{PBT}$ (a) collapses to a single mode, while $\mathrm{P2BT}$  (b) continues to explore the space at later stages of optimization.}}
    \label{fig:schedules}
    \end{minipage}
\end{figure}
\end{comment}

\section{Experiment Details}
\label{appendix:details}

For all experiments we set $\beta_t = c_1 + \text{log}(c_2t)$ with $c_1 = 0.2$ and $c_2 = 0.4$, as in the traffic speed data experiment from \cite{bogunovic2016time}. 

In Table \ref{table:impala_fixed}, \ref{table:ppo_fixed} \& \ref{table:cifar_fixed} and we show hyperparameters for the IMPALA, PPO and CIFAR experiments. In Table \ref{table:impala_learned}, \ref{table:ppo_ranges} and \ref{table:ppo_ranges} we show the bounds for the hyperparameters learned by PBT and PB2. All methods were initialized by randomly sampling from these bounds.

\begin{table}[h!]
    \centering
    \begin{minipage}{.45\linewidth}
        \caption{\footnotesize{IMPALA: Fixed}}
        \label{table:impala_fixed}
        \centering
        \scalebox{0.85}{
        \begin{tabular}{l*{2}{c}r}
        \toprule
        \textbf{Parameter} & \textbf{Value}  \\
        \midrule
        Num Workers & 5  \\
        Num GPUs & 0  \\
        \bottomrule
        \\
    \end{tabular}}
    \end{minipage}
    \begin{minipage}{.45\linewidth}
    \caption{\footnotesize{IMPALA: Learned}}
    \label{table:impala_learned}
    \centering
    \scalebox{0.85}{
    \begin{tabular}{l*{2}{c}r}
    \toprule
    \textbf{Parameter} & \textbf{Value}  \\
    \midrule
    Epsilon & $\{0.01, 0.5\}$  \\
    Learning Rate & $\{10^{-3}, 10^{-5}\}$  \\
    Entropy Coeff & $\{0.001, 0.1\}$ \\
    \bottomrule
    \\
    \end{tabular}}
    \end{minipage}
\label{table:impala_learned}
\end{table}

\begin{table}[h!]
    \centering
    \begin{minipage}{.45\linewidth}
        \caption{\footnotesize{PPO: Fixed}}
        \label{table:ppo_fixed}
        \centering
        \scalebox{0.85}{
        \begin{tabular}{l*{2}{c}r}
        \toprule
        \textbf{Parameter} & \textbf{Value}  \\
        \midrule
        Filter & $MeanStdFilter$  \\
        SGD Iterations & 10  \\
        Architecture & 32-32 \\
        $\mathrm{ready}$ & $5 \times 10^4$  \\
        \bottomrule
        \\
    \end{tabular}}
    \end{minipage}
    \begin{minipage}{.45\linewidth}
    \caption{\footnotesize{PPO: Learned}}
    \label{table:ppo_ranges}
    \centering
    \scalebox{0.85}{
    \begin{tabular}{l*{2}{c}r}
    \toprule
    \textbf{Parameter} & \textbf{Value}  \\
    \midrule
    Batch Size & $\{1000, 60000\}$  \\
    GAE $\lambda$ & $\{0.9, 0.99\}$  \\
    PPO Clip $\epsilon$ & $\{0.1, 0.5\}$ \\
    Learning Rate $\eta$ & $\{10^{-3}, 10^{-5}\}$  \\
    \bottomrule
    \\
    \end{tabular}}
    \end{minipage}
\end{table}

The model used for the CIFAR dataset was from: \url{https://zhenye-na.github.io/2018/09/28/pytorch-cnn-cifar10.html}. All experiments were run using a $32$ core machine. 

\begin{table}[h!]
    \centering
    \begin{minipage}{.45\linewidth}
        \caption{\footnotesize{CIFAR: Fixed}}
        \label{table:cifar_fixed}
        \centering
        \scalebox{0.85}{
        \begin{tabular}{l*{2}{c}r}
        \toprule
        \textbf{Parameter} & \textbf{Value}  \\
        \midrule
        Optimizer & $\mathrm{Adam}$  \\
        Iterations & 50  \\
        Architecture & 3 Conv Layers \\
        $\mathrm{ready}$ & $5$  \\
        \bottomrule
        \\
    \end{tabular}}
    \end{minipage}
    \begin{minipage}{.45\linewidth}
    \caption{\footnotesize{CIFAR: Learned}}
    \centering
    \scalebox{0.85}{
    \begin{tabular}{l*{2}{c}r}
    \toprule
    \textbf{Parameter} & \textbf{Value}  \\
    \midrule
    Train Batch Size & $\{4, 128\}$  \\
    Dropout-1 & $\{0.1, 0.5\}$  \\
    Dropout-2 & $\{0.1, 0.5\}$  \\
    Learning Rate & $\{10^{-3}, 10^{-4}\}$  \\
    Weight Decay & $\{10^{-3}, 10^{-5}\}$  \\
    Momentum & $\{0.8, 0.99\}$  \\
    \bottomrule
    \\
    \end{tabular}}
    \end{minipage}
\label{table:cifar_learned}
\end{table}

%% file: TVGP_Theoretical_Analysis_2.tex
\section{Theoretical Results}
\label{sec:theory}

We show the derivation for Lemma \ref{lem:maxreward_minregret}. 
\begin{proof}
We have a reward at the starting iteration $F_1(x_1)$ as a constant that allows us to write the objective function as:
\begin{align}
    F_T(x_T) - F_1(x_1) &=  F_T(x_T) - F_{T-1}(x_{T-1})+  \dots+ F_{3}(x_{3}) - F_{2}(x_{2}) + F_2(x_2) - F_1(x_1) 
    \label{eq:equivalent_reward_regret}
\end{align}

Therefore, maximizing the left of Eq. (\ref{eq:equivalent_reward_regret}) is equivalent to minimizing the cummulative regret as follows:
\begin{align}
    \max \left[ F_T(x_T) - F_1(x_1) \right] &=  \max \sum_{t=1}^T F_t(x_t) - F_{t-1}(x_{t-1}) = \max \sum_{t=1}^T f_t(x_t) = \min \sum_{t=1}^T r_t(x_t)\nonumber
\end{align}
where we define $f_t(x_t)=F_t(x_t)-F_{t-1}(x_{t-1})$, the regret $r_t=f_t(x^*_t)-f_t(x_t)$ and $f_t(x_t^*):= \max_{\forall x} f_t(x)$ is an unknown constant.
%Therefore minimizing regret is equivalent to maximizing final performance.
\end{proof}

\subsection{Convergence Analysis}

We minimize the cumulative regret $R_{T}$  by sequentially
suggesting an $\bx_{t}$ to be used in each iteration $t$. We shall derive
the upper bound in the cumulative regret and show that it asymptotically
goes to zero as $T$ increases, i.e., $\lim_{T\rightarrow\infty}\frac{R_{T}}{T}=0$. We make the following smoothness assumption to derive the regret bound
of the proposed algorithm.

\textbf{Assumptions.} We will assume that the kernel $k$ holds for some $\left(a,b\right)$ and
$\forall L_t\ge0$. The joint kernel satisfies for all dimensions $j=1,...,d,$
\begin{align}
\forall L_t & \ge0,t\le\mathcal{T},p(\sup\left|\frac{\partial f_{t}(\bx)}{\partial\bx^{(j)}}\right|\ge L_t)\le ae^{-\left(L_t/b\right)^{2}}.\label{eq:Lipschitz_assumption}
\end{align}
These assumptions are achieved by using a time-varying kernel $k_{time}(t,t')=\left(1-\omega\right)^{\frac{|t-t'|}{2}}$
\cite{bogunovic2016time} with the smooth functions \cite{bogunovic2016time}. For completeness, we restate Lemma \ref{lem:Lipschitz_bound}, \ref{lem:bound_discretization}, \ref{lem:bound_fx_mux}, \ref{lem:bound_MIG_timevarying} from \cite{gp_ucb, bogunovic2016time}, then present our new theoretical results in Lemma  \ref{lem:bound_UCB_US}, \ref{lem:The-sum-variance_UCB}, \ref{lem:bound_sum_sigma} and Theorem \ref{main_theorem}.

\begin{lemma} [\cite{gp_ucb}]
\label{lem:Lipschitz_bound}Let $L_{t}=b\sqrt{\log3da\frac{\pi_{t}}{\delta}})$
where $\sum_{t=1}^{T}\frac{1}{\pi_{t}}=1$, we have with probability
$1-\frac{\delta}{3},$
\begin{align}
\left|f_{t}(\bx)-f_{t}(\bx')\right| & \le L_{t}||\bx-\bx'||_{1},\forall t,\bx,\bx'\in D.\label{eq:Lipschitz_Lt}
\end{align}
\end{lemma}

% \begin{proof}
% Using the Lipschitz assumption in Eqn (\ref{eq:Lipschitz_assumption})
% and the union bound over $j=1,...,d$, the event corresponding to
% a single time in Eqn (\ref{eq:Lipschitz_Lt}) will hold with the probability
% at least $1-da\exp\left(-\frac{L_{t}^{2}}{b^{2}}\right)=1-\frac{\delta}{3\pi_{t}}$.
% Taking the union bound over $t$, we get $1-\sum_{t=1}^{T}da\exp\left(-\frac{L_{t}^{2}}{b^{2}}\right)=1-\sum_{t=1}^{T}\frac{\delta}{3\pi_{t}}=1-\frac{\delta}{3}$
% since $\sum_{t=1}^{T}\frac{1}{\pi_{t}}=1$.
% \end{proof}
%Below we define a discretization which serves as a tool for proving the theorem, but it is not used in the implementation.

\begin{lemma} [\cite{gp_ucb}]
\label{lem:bound_discretization}We define a discretization $D_{t}\subset D\subseteq[0,r]^{d}$
of size $(\tau_{t})^{d}$ satisfying $||\bx-\left[\bx\right]_{t}||_{1}\le\frac{d}{\tau_{t}},\forall\bx\in D$
where $\left[\bx\right]_{t}$ denotes the closest point in $D_{t}$
to $\bx$. By choosing $\tau_{t}=\frac{t^{2}}{L_{t}d}=rdbt^{2}\sqrt{\log\left(3da\pi_{t}/\delta\right)}$,
we have 
\begin{align*}
|f_{t}(\bx)-f_{t}(\left[\bx\right]_{t})| & \le\frac{1}{t^{2}}.
\end{align*}
\end{lemma}

% \begin{proof}
% Using the discretization assumption and the Lipschitz continuity in
% Lem. \ref{lem:Lipschitz_bound}, we have
% \begin{align*}
% \left|f_{t}(\bx)\text{\textminus}f_{t}(\left[\bx\right]_{t})\right| & \le L_{t}||\bx-\left[\bx\right]_{t}||_{1}\le\frac{L_{t}d}{\tau_{t}}
% \end{align*}
% By setting $\tau_{t}=\frac{t^{2}}{L_{t}d}$, we have that $\left|f_{t}(\bx)-f_{t}(\left[\bx\right]_{t})\right|\le\frac{1}{t^{2}}$.
% \end{proof}
%We note that a uniformly-spaced grid suffices to ensure the discretization holds.
\begin{lemma}
[\cite{gp_ucb}]
\label{lem:bound_fx_mux}Let $\beta_{t}\ge2\log\frac{3\pi_{t}}{\delta}+2d\log\left(rdbt^{2}\sqrt{\log\frac{3da\pi_{t}}{2\delta}}\right)$
where $\sum_{t=1}^{T}\pi_{t}^{-1}=1$, then with probability at least
$1-\frac{\delta}{3}$, we have
\begin{align*}
|f_{t}(\bx_{t})-\mu_{t}(\bx_{t})|\le & \sqrt{\beta_{t}}\sigma_{t}(\bx_{t}),\forall t,\forall\bx\in\mathcal{D}.
\end{align*}
\end{lemma}

\begin{proof}
We note that conditioned on the outputs $(y_{1},...,y_{t-1})$, the
sampled points $(\bx_{1},...,\bx_{t})$ are deterministic, and $f_{t}(\bx_{t})\sim\mathcal{N}\left(\mu_{t}(\bx_{t}),\sigma_{t}^{2}(\bx_{t})\right)$.
Using Gaussian tail bounds \cite{wainwright2015basic}, a random
variable $f\sim\mathcal{N}(\mu,\sigma^{2})$ is within $\sqrt{\beta}\sigma$
of $\mu$ with probability at least $1-\exp\left(-\frac{\beta}{2}\right)$.
 Therefore, we first claim that if $\beta_{t}\ge2\log\frac{3\pi_{t}}{\delta}$
then the selected points $\left\{ \bx_{t}\right\} _{t=1}^{T}$ satisfy
the confidence bounds with probability at least $1-\frac{\delta}{3}$
\begin{align}
\left|f_{t}(\bx_{t})-\mu_{t}(\bx_{t})\right| & \le\sqrt{\beta_{t}}\sigma_{t}(\bx_{t}),\forall t.\label{eq:bound_selected_xt}
\end{align}

This is true because the confidence bound for individual $\bx_{t}$
will hold with probability at least $1-\frac{\delta}{3\pi_{t}}$ and
taking union bound over $\forall t$ will lead to $1-\frac{\delta}{3}$.

We show above that the bound is applicable for the selected points
$\left\{ \bx_{t}\right\} _{t=1}^{T}$. To ensure that the bound is
applicable for all points in the domain $D_{t}$ and $\forall t$,
we can set $\beta_{t}\ge2\log\frac{3|D_{t}|\pi_{t}}{\delta}$ where
$\sum_{t=1}^{T}\pi_{t}^{-1}=1$ e.g., $\pi_{t}=\frac{\pi^{2}t^{2}}{6}$
\begin{align}
p(\left|f_{t}\left(\bx_{t}\right)-\mu_{t}\left(\bx_{t}\right)\right| & \le\sqrt{\beta_{t}}\sigma_{t}\left(\bx_{t}\right)\ge1-|D_{t}|\sum_{t=1}^{T}\exp\left(-\beta_{t}/2\right)=1-\frac{\delta}{3.}.\label{eq:bound_all_Dt}
\end{align}

By discretizing the domain $D_{t}$ in Lem. \ref{lem:bound_discretization},
we have a connection to the cardinality of the domain that $|D_{t}|=(\tau_{t})^{d}=\left(rdbt^{2}\sqrt{\log\left(3da\pi_{t}/\delta\right)}\right)^{d}$.
 Therefore, we need to set $\beta_{t}$ such that both conditions
in Eq. (\ref{eq:bound_selected_xt}) and Eq. (\ref{eq:bound_all_Dt})
are satisfied. We simply take a sum of them and get $\beta_{t}\ge2\log\frac{3\pi_{t}}{\delta}+2d\log\left(rdbt^{2}\sqrt{\log\frac{3da\pi_{t}}{2\delta}}\right)$.
\end{proof}

%\paragraph{Mutual Information:}

We use $TB$ to denote the batch setting where we will run the algorithm over $T$ iterations with a batch size $B$. The mutual information is defined
as $\tilde{\idenmat}(f_{TB};y_{TB})=\frac{1}{2}\log\det\left(\idenmat_{TB}+\sigma_{f}^{-2}\tilde{K}_{TB}\right)$
and the maximum information gain is as $\tilde{\gamma}_{T}:=\max\tilde{\idenmat}({\bf f}_{TB};\by_{TB})$
where ${\bf f}_{TB}:=f_{TB}(\bx_{TB})=\left(f_{t,b}(\bx_{t,b}),...,f_{T,B}(\bx_{T,B})\right),\forall b=1....B,\forall t=t...T$
for the time variant GP $f$. Using the result presented in \citep{bogunovic2016time}, we can adapt the bound on the time-varying information gain into the parallel setting using a population size of $B$ below. %We will show that $\tilde{\gamma}_{TB}\le\left(\frac{T}{\tilde{N}B}+1\right)\left(\gamma_{\tilde{N}B}+\sigma_{f}^{-2}\left[\tilde{N}B\right]^{3}\omega\right)$. To derive the bound over the maximum information gain, we shall split $T\times B$ observations in time steps $\left\{ 1,...,T\right\} $ into $T/\tilde{N}$ blocks of length $\tilde{N}\times B$, such that within each block the function $f_{i}$ does not vary significantly. We assume for the time being that $T/\tilde{N}$ is an integer, and then handle the general case.

\begin{lemma}
\label{lem:bound_MIG_timevarying} (adapted from \cite{bogunovic2016time} with a batch size $B$) Let $\omega$ be the forgetting-remembering
trade-off parameter and consider the kernel for time $1-K_{time}(t,t')\le\omega\left|t-t'\right|$,
we bound the maximum information gain that
\begin{align*}
\tilde{\gamma}_{TB} & \le\left(\frac{T}{\tilde{N}\times B}+1\right)\left(\gamma_{\tilde{N}\times B}+\sigma_{f}^{-2}\left[\tilde{N}\times B\right]^{3}\omega\right).
\end{align*}
\end{lemma}

\paragraph{Uncertainty Sampling (US).}

We next derive an upper bound over the maximum information gain obtained
from a batch $\bx_{t,b},\forall b=1,...,B$. In other words, we want
to show that the information gain by our chosen points $\bx_{t,b}$
will not go beyond the ones by maximizing the uncertainty.
For this, we define an uncertainty sampling (US) scheme which fills
in a batch $\bx_{t,b}^{\textrm{US}}$ by maximizing the GP predictive
variance. Particularly, at iteration $t$, we select $\bx_{t,b}^{\textrm{US}}=\arg\max_{\bx}\sigma_{t}(\bx\mid D_{t,b-1}),\forall b\le B$
and the data set is augmented over time to include the information
of the new point, $D_{t,b}=D_{t,b-1}\cup\bx_{t,b}^{\textrm{US}}$.
We note that we use $\bx_{t,b}^{\textrm{US}}$ to derive the upper
bound, but this is not used by our PB2 algorithm.
\begin{lemma}
\label{lem:bound_UCB_US}Let $\bx_{t,b}^{\textrm{PB2}}$ be the point chosen
by our algorithm and $\bx_{t,b}^{\textrm{US}}$ be the point chosen
by uncertainty sampling (US) by maximizing the GP predictive variance
$\bx_{t,b}^{\textrm{US}}=\arg\max_{\bx\in D}\sigma_{t}(\bx\mid D_{t,b-1}),\forall b=1,...B$
and $D_{t,b}=D_{t,b-1}\cup\bx_{t,b}$. We have
\begin{align*}
\sigma_{t+1,1}\left(\bx_{t+1,1}^{\textrm{PB2}}\right)\le\sigma_{t+1,1}\left(\bx_{t+1,1}^{\textrm{US}}\right) & \le\sigma_{t,b}\left(\bx_{t,b}^{\textrm{US}}\right),\forall t\in\left\{ 1,...,T\right\} ,\forall b\in\left\{ 1,...B\right\}.
\end{align*}
\end{lemma}

\begin{proof}
The first inequality is straightforward that the point chosen by uncertainty
sampling will have the highest uncertainty $\sigma_{t+1,1}\left(\bx_{t+1,1}^{\textrm{PB2}}\right)\le\sigma_{t+1,1}\left(\bx_{t+1,1}^{\textrm{US}}\right)=\arg\max_{\bx}\sigma_{t}(\bx\mid D_{t,b-1})$.

The second inequality is obtained by using the principle of \textquotedblleft information
never hurts\textquotedblright{} \cite{krause2008near}, we know that
the GP uncertainty for all locations $\forall\bx$ decreases with
observing a new point. Therefore, the uncertainty at the future iteration
$\sigma_{t+1}$ will be smaller than that of the current
iteration $\sigma_{t}$, i.e., $\sigma_{t+1,b}\left(\bx_{t+1,b}^{\textrm{US}}\right)\le\sigma_{t,b}\left(\bx_{t,b}^{\textrm{US}}\right),\forall b\le B,\forall t\le T$.
We thus conclude the proof $\sigma_{t+1,1}\left(\bx_{t+1,1}^{\textrm{US}}\right)\le\sigma_{t,b}\left(\bx_{t,b}^{\textrm{US}}\right),\forall t\in\left\{ 1,...,T\right\} ,\forall b\in\left\{ 1,...B\right\} $.
\end{proof}

\begin{lemma}
\label{lem:The-sum-variance_UCB}The sum of variances of the points
selected by the our PB2 algorithm $\sigma()$ is bounded by the sum
of variances by uncertainty sampling $\sigma^{\textrm{US}}()$. Formally, w.h.p.,
$\sum_{t=2}^{T}\sigma_{t,1}\left(\bx_{t,1}\right) \le\frac{1}{B}\sum_{t=1}^{T}\sum_{b=1}^{B}\sigma_{t,b}\left(\bx_{t,b}^{\textrm{US}}\right)$.
\end{lemma}

\begin{proof}
By the definition of uncertainty sampling in Lem. \ref{lem:bound_UCB_US},
we have $\sigma_{t+1,1}\left(\bx_{t+1,1}\right)\le\sigma_{t,b}\left(\bx_{t,b}^{\textrm{US}}\right),\forall t\in\left\{ 1,...,T\right\} ,\forall b\in\left\{ 2,...B\right\} $
and $\sigma_{t,1}\left(\bx_{t,1}\right)\le\sigma_{t,1}\left(\bx_{t,1}^{\textrm{US}}\right)$
where $\bx_{t,1}$ is the point chosen by our PB2 and $\bx_{t,1}^{\textrm{US}}$
is from uncertainty sampling. Summing all over $B$, we obtain
\begin{align*}
\sigma_{t,1}\left(\bx_{t,1}\right)+\left(B-1\right)\sigma_{t+1,1}\left(\bx_{t+1,1}\right) & \le\sigma_{t,1}\left(\bx_{t,1}^{\textrm{US}}\right)+\sum_{b=2}^{B}\sigma_{t,b}\left(\bx_{t,b}^{\textrm{US}}\right)\\
\sum_{t=1}^{T}\sigma_{t,1}\left(\bx_{t,1}\right)+\left(B-1\right)\sum_{t=1}^{T}\sigma_{t+1,1}\left(\bx_{t+1,1}\right) & \le\sum_{t=1}^{T}\sum_{b=1}^{B}\sigma_{t,b}\left(\bx_{t,b}^{\textrm{US}}\right)\,\,\,\,\,\,\,\,\,\,\,\textrm{by}\textrm{\,summing}\thinspace\textrm{over}\thinspace T\\
\sum_{t=2}^{T}\sigma_{t,1}\left(\bx_{t,1}\right) & \le\frac{1}{B}\sum_{t=1}^{T}\sum_{b=1}^{B}\sigma_{t,b}\left(\bx_{t,b}^{\textrm{US}}\right).
\end{align*}
The last equation is obtained because of $\sigma_{1,1}\left(\bx_{1,1}\right)\ge0$
and $\left(B-1\right)\sigma_{T+1,1}\left(\bx_{T+1,1}\right)\ge0$.
\end{proof}
\begin{lemma}
\label{lem:bound_sum_sigma}Let $C_{1}=\frac{32}{\log\left(1+\sigma_{f}^{-2}\right)}$,
$\sigma_{f}^{2}$ be the measurement noise variance and $\tilde{\gamma}_{TB}:=\max\tilde{\idenmat}$
be the maximum information gain of time-varying kernel, we have $\sum_{t=1}^{T}\sum_{b=1}^{B}\sigma_{t,b}^{2}(\bx_{t,b}^{\textrm{US}})\le\frac{C_{1}}{16}\tilde{\gamma}_{TB}$
where $\bx_{t,b}^{\textrm{US}}$ is the point selected by uncertainty
sampling (US).
\end{lemma}

\begin{proof}
We show that $\sigma_{t,b}^{2}(\bx_{t,b}^{\textrm{US}})=\sigma_{f}^{2}\left(\sigma_{f}^{-2}\sigma_{t,b}^{2}(\bx_{t,b}^{\textrm{US}})\right)\le\sigma_{f}^{2}C_{2}\log\left(1+\sigma_{f}^{-2}\sigma_{t,b}^{2}\left(\bx_{t,b}^{\textrm{US}}\right)\right),\forall b\le B,\forall t\le T$
where $C_{2}=\frac{\sigma_{f}^{-2}}{\log\left(1+\sigma_{f}^{-2}\right)}\ge1$
and $\sigma_{f}^{2}$ is the measurement noise variance. We have the
above inequality because $s^{2}\le C_{2}\log\left(1+s^{2}\right)$
for $s\in\left[0,\sigma_{f}^{-2}\right]$ and $\sigma_{f}^{-2}\sigma_{t,b}^{2}\left(\bx_{t,b}^{\textrm{US}}\right)\le\sigma^{-2}k\left(\bx_{t,b}^{\textrm{US}},\bx_{t,b}^{\textrm{US}}\right)\le\sigma_{f}^{-2}$.
We then use Lemma 5.3 of \cite{gp_bucb} to have the
information gain over the points chosen by a time-varying kernel $\tilde{\idenmat}=\frac{1}{2}\sum_{t=1}^{T}\sum_{b=1}^{B}\log\left(1+\sigma_{f}^{-2}\sigma_{t,b}^{2}\left(\bx_{t,b}^{\textrm{US}}\right)\right)$.
Finally, we obtain
\begin{align*}
\sum_{t=1}^{T}\sum_{b=1}^{B}\sigma_{t,b}^{2}(\bx_{t,b}^{\textrm{US}}) & \le\sigma_{f}^{2}C_{2}\sum_{t=1}^{T}\sum_{b=1}^{B}\log\left(1+\sigma_{f}^{-2}\sigma_{t,b}^{2}\left(\bx_{t,b}^{\textrm{US}}\right)\right)=2\sigma_{f}^{2}C_{2}\tilde{\idenmat}=\frac{C_{1}}{16}\tilde{\gamma}_{TB}
\end{align*}
where $C_{1}=\frac{2}{\log\left(1+\sigma_{f}^{-2}\right)}$ and $\tilde{\gamma}_{TB}:=\max\tilde{\idenmat}$
is the definition of maximum information gain given by $T\times B$
data points from a GP for a specific time-varying kernel.
\end{proof}
\begin{theorem} \label{main_theorem}
Let the domain $\mathcal{D}\subset[0,r]^{d}$ be compact and convex
where $d$ is the dimension and suppose that the kernel is such that
$f\sim GP(0,k)$ is almost surely continuously differentiable and
satisfies Lipschitz assumptions for some $a,b$. Fix $\delta\in(0,1)$
and set $\beta_{T}=2\log\frac{\pi^{2}T^{2}}{2\delta}+2d\log rdbT^{2}\sqrt{\log\frac{da\pi^{2}T^{2}}{2\delta}}$.
Defining $C_{1}=32/\log(1+\sigma_{f}^{2})$, the PB2 algorithm satisfies
the following regret bound after $T$ time steps:
\begin{align*}
R_{TB}=\sum_{t=1}^{T} f_{t}(\bx_{t}^{*})-\max_{b=1,...,B} f_{t}(\bx_{t,b})\le & \sqrt{C_{1}T\beta_{T}\left(\frac{T}{\tilde{N}B}+1\right)\left(\gamma_{\tilde{N}B}+\left[\tilde{N}B\right]^{3}\omega\right)}+2
\end{align*}
with probability at least $1-\delta$, the bound holds for any
$\tilde{N}\in\left\{ 1,...,T\right\} $ and $B\ll T$.
\end{theorem}

\begin{proof}
Let $\bx_{t}^{*}=\arg\max_{\forall\bx}f_{t}(\bx)$ and $\bx_{t,b}$
be the point chosen by our algorithm at iteration $t$ and batch element
$b$, we define the (time-varying) instantaneous regret as $r_{t,b}=f_{t}(\bx_{t}^{*})-f_{t}(\bx_{t,b})$
and the (time-varying) batch instantaneous regret over $B$ points
is as follows
\begin{align}
r_{t}^{B} & =\min_{b\le B}r_{t,b}=\min_{b\le B}f_{t}(\bx_{t}^{*})-f_{t}(\bx_{t,b}),\forall b\le B\nonumber \\
 & \le f_{t}(\bx_{t}^{*})-f_{t}(\bx_{t,1})\le\text{\ensuremath{\mu_{t}(\bx_{t}^{*})+\sqrt{\kappa_{t}}\sigma_{t}(\bx_{t}^{*})+\frac{1}{t^{2}}-f_{t}(\bx_{t,1})}}\,\,\,\,\,\,\,\,\,\,\,\,\,\,\,\,\,\,\,\,\,\,\,\,\,\,\,\,\,\,\textrm{by}\,\textrm{Lem}.\,\ref{lem:bound_discretization}\nonumber \\
 & \le\mu_{t}(\bx_{t,1})+\sqrt{\kappa_{t}}\sigma_{t}(\bx_{t,1})+\frac{1}{t^{2}}-f_{t}(\bx_{t,1})\le2\sqrt{\kappa_{t}}\sigma_{t}(\bx_{t,1})+\frac{1}{t^{2}}\label{eq:bound_rt}
\end{align}
where we have used the
property that $\mu_{t}(\bx_{t,1})+\sqrt{\beta_{t}}\sigma_{t}(\bx_{t,1})\ge\mu_{t}(\bx_{t}^{*})+\sqrt{\beta_{t}}\sigma_{t}(\bx_{t}^{*})$
by the definition of selecting $\bx_{t,1}=\arg\max_{\bx}\mu_{t}(\bx)+\sqrt{\beta_{t}}\sigma_{t}(\bx)$. Next, we bound the cumulative batch regret as
\begin{align}
R_{TB} & = \sum_{t=1}^{T}r_{t}^{B}  \le\sum_{t=1}^{T}\left(2\sqrt{\kappa_{t}}\sigma_{t}(\bx_{t,1})+\frac{1}{t^{2}}\right)\,\,\,\,\,\,\,\,\,\,\,\,\,\,\,\,\,\,\,\,\,\,\,\,\,\,\,\,\,\,\,\,\,\,\,\,\,\,\,\,\,\,\,\,\,\,\,\,\,\,\,\,\,\,\,\,\,\,\,\,\,\,\,\,\,\,\,\,\,\,\,\,\,\,\,\,\,\,\,\,\,\,\,\,\,\,\,\,\,\,\,\,\,\,\textrm{by\,Eq.\,(\ref{eq:bound_rt})}\nonumber \\
 & \le2\sqrt{\kappa_{T}}\sigma_{1}(\bx_{1,1})+\frac{2\sqrt{\kappa_{T}}}{B}\sum_{t=1}^{T}\sum_{b=1}^{B}\sigma_{t,b}\left(\bx_{t,b}^{\textrm{US}}\right)+\sum_{t=1}^{T}\frac{1}{t^{2}}\,\,\,\,\,\,\,\,\textrm{by}\,\textrm{Lem}.\,\ref{lem:bound_UCB_US}\thinspace\textrm{and\thinspace\ensuremath{\kappa_{T}\ge\kappa_{t},\forall t}\ensuremath{\le T}}\nonumber \\
 & \le\frac{4\sqrt{\kappa_{T}}}{B}\sum_{t=1}^{T}\sum_{b=1}^{B}\sigma_{t,b}\left(\bx_{t,b}^{\textrm{US}}\right)+\sum_{t=1}^{T}\frac{1}{t^{2}}\label{eq:R_T_two_terms_split}\\
 & \le\frac{4}{B}\sqrt{\kappa_{T}\times TB\sum_{t=1}^{T}\sum_{b=1}^{B}\sigma_{t,b}^{2}(\bx_{t,b}^{\textrm{US}})}+2\le\sqrt{C_{1}\frac{T}{B}\kappa_{T}\tilde{\gamma}_{TB}}+2\label{eq:R_T_step2}\\
 & \le\sqrt{C_{1}\frac{T}{B}\kappa_{T}\left(\frac{T}{\tilde{N} B}+1\right)\left(\gamma_{\tilde{N}B}+\frac{1}{\sigma_{f}^{2}}\left[\tilde{N} B\right]^{3}\omega\right)}+2\label{eq:R_T_last_step}
\end{align}
where $C_{1}=32/\log(1+\sigma_{f}^{2})$, $\bx_{t,b}^{\textrm{US}}$
is the point chosen by uncertainty sampling -- used to provide the
upper bound in the uncertainty. In Eq. (\ref{eq:R_T_two_terms_split}),
we take the upper bound by considering two possible cases: either
$\sigma_{1}(\bx_{1,1})\ge\frac{1}{B}\sum_{t=1}^{T}\sum_{b=1}^{B}\sigma_{t,b}\left(\bx_{t,b}^{\textrm{US}}\right)$
or $\frac{1}{B}\sum_{t=1}^{T}\sum_{b=1}^{B}\sigma_{t,b}\left(\bx_{t,b}^{\textrm{US}}\right)\ge\sigma_{1}(\bx_{1,1})$.
It results in $\frac{2}{B}\sum_{t=1}^{T}\sum_{b=1}^{B}\sigma_{t,b}\left(\bx_{t,b}^{\textrm{US}}\right)\ge\frac{1}{B}\sum_{t=1}^{T}\sum_{b=1}^{B}\sigma_{t,b}\left(\bx_{t,b}^{\textrm{US}}\right)+\sigma_{1}(\bx_{1,1})$.
In Eq. (\ref{eq:R_T_step2}) we have used $\sum_{t=1}^{\infty}\frac{1}{t^{2}}\le\pi^{2}/6\le2$
and $||z||_{1}\le\sqrt{T}||z||_{2}$ for any vector $z\in\mathcal{R}^{T}$.
In Eq. (\ref{eq:R_T_last_step}), we utilize Lem. \ref{lem:bound_MIG_timevarying}.

Finally, given the squared exponential (SE) kernel defined, $\gamma_{\tilde{N}B}^{SE}=\mathcal{O}(\left[\log\tilde{N}B\right]^{d+1})$,
the bound is $R_{TB}\le\sqrt{C_{1}\frac{T}{B}\beta_{T}\left(\frac{T}{\tilde{N}B}+1\right)\left((d+1)\log\left(\tilde{N}B\right)+\frac{1}{\sigma_{f}^{2}}\left[\tilde{N}B\right]^{3}\omega\right)}+2$
where $\tilde{N}\le T$ and $B\ll T$. %We refer to Theorem 5 in \cite{gp_ucb} for the bound of other common kernels. We note that this bound $R_{TB}$ in Eq. (\ref{eq:R_T_last_step}) is tighter than in the original paper \cite{bogunovic2016time} due to the $\tilde{\gamma}_{T}$ derived in Lem. \ref{lem:bound_MIG_timevarying} (ours $\tilde{N}^{5/2}\le\tilde{N}^{3}$ \cite{bogunovic2016time}).
\end{proof}

In our time-varying setting, if the time-varying function is highly correlated, i.e., the information between $f_1(.)$ and $f_T(.)$ does not change significantly, we have  $\omega \rightarrow 0$ and $\tilde{N} \rightarrow T$. Then, the regret bound grows sublinearly with the number of iterations $T$, i.e., $\lim_{T \rightarrow\infty}\frac{R_{TB}}{T}=0$. This bound suggests
that the gap between $f_{t}(\bx_{t})$ and the optimal $f_{t}(\bx_{t}^{*})$ vanishes asymptotically using PB2. In addition, our regret bound is tighter
and better with increasing batch size $B$.

On the other hand in the worst case, if the time-varying function is not correlated, such as $\tilde{N} \rightarrow 1$ and $ \omega \rightarrow 1$, then PB2 achieves the linear regret \citep{bogunovic2016time}.